\pdfoutput=1
\documentclass[letterpaper, 10 pt, conference]{ieeeconf}  % Comment this line out if you need a4paper\usepackage[utf8]{inputenc}

\usepackage{amsmath, amsfonts, amssymb, mathtools, commath, hyperref, cancel, bm, xcolor}
\usepackage{cite, microtype}

\IEEEoverridecommandlockouts                              % This command is only needed if 
\title{\LARGE \bf
Structural Constraints for Physics-augmented Learning
}

\author{Simon Kuang$^{1}$ and Xinfan Lin$^{1}$% <-this % stops a space
\thanks{*This work was supported by the National Science Foundation
CAREER Program (Grant No. 2046292)}% <-this % stops a space
\thanks{$^{1}$ Department of Mechanical and Aerospace Engineering, University of California, Davis; Davis, CA 95616, USA. {\ttfamily \{slku, lxflin\}@ucdavis.edu}}}%

\newtheorem{proposition}{Proposition}
\newtheorem{theorem}{Theorem}
\newtheorem{definition}{Definition}
\newtheorem{fact}{Obvious Fact}
\DeclareMathOperator*{\argmin}{arg\,min}

\begin{document}

\maketitle
\thispagestyle{empty}
\pagestyle{empty}

%%%%%%%%%%%%%%%%%%%%%%%%%%%%%%%%%%%%%%%%%%%%%%%%%%%%%%%%%%%%%%%%%%%%%%%%%%%%%%%% 
\begin{abstract}
When the physics is wrong, physics-informed machine learning becomes physics-\emph{misinformed} machine learning.
A powerful black-box model should not be able to conceal misconceived physics.
We propose two criteria that can be used to assert integrity that a hybrid (physics plus black-box) model:
0) the black-box model should be unable to replicate the physical model, and 1) any best-fit hybrid model has the same physical parameter as a best-fit standalone physics model.
We demonstrate them for a sample nonlinear mechanical system approximated by its small-signal linearization.
\end{abstract}

%%%%%%%%%%%%%%%%%%%%%%%%%%%%%%%%%%%%%%%%%%%%%%%%%%%%%%%%%%%%%%%%%%%%%%%%%%%%%%%%
\section{Introduction}
Universal function approximation is customary for physical-augmented machine learning (PAML, a sub-discipline of physics-\emph{informed} ML) \cite{chinchilla_learning_2023},
also known as discrepancy modeling \cite{wohlleben_development_2022,kaheman_learning_2019},
reduced-order model closure \cite{ahmed_closures_2021},
residual modeling \cite{kennedy_bayesian_2001,saveriano_data-efficient_2017,chinchilla_learning_2023},
or universal differential equation \cite{rackauckas_universal_2021}.
In the context of dynamic systems, PAML combines a black-box
regression model (henceforth called the residual model)
with an evolution law derived from physical first principles such as conservation laws.

We use ``black-box'' in the wide sense that ``No physical insight is available or used, but the chosen model structure belongs to families that are known to have good flexibility and have been `successful in the past'\,''\cite{sjoberg_nonlinear_1995}.
The definition and usefulness  of ``physical'' comprise PAML's raison d'être and should not be taken for granted.
A critical examination follows in \S\ref{subsec:motivation}.

\subsection{Model structure}
Usually the residual model is drawn from a class of universal approximators.
Examples include decision tree ensemble \cite{wohlleben_development_2022},
kernel machine \cite{kennedy_bayesian_2001,saveriano_data-efficient_2017,peli_physics-based_2022},
Koopman lift \cite{kaheman_learning_2019,ebers_discrepancy_2023},
generalized linear regression \cite{wohlleben_transferability_2023},
% linear-time invariant system \cite{galgali_residual_2023,wang_learning_2023},
or (deep, recursive) neural network \cite{shi_neural_2019,belbute-peres_combining_2020,rackauckas_universal_2021,yin_augmenting_2021,dona_constrained_2022,levine_framework_2022}.

There are three ways to combine a physical and residual model.
Some physics-informed learning methods use the physical model as a feature for a neural network \cite{belbute-peres_combining_2020} that has the final say on predictions.
Otherwise, there are two options for to injecting black-box residual model into a physical model:
\begin{itemize}
        \item 
        Sum the residual model dynamics with the physical model dynamics to generate the time evolution of the physical state \cite{saveriano_data-efficient_2017,shi_neural_2019,manzi_discovering_2020,ahmed_closures_2021,rackauckas_universal_2021,yin_augmenting_2021,dona_constrained_2022,wohlleben_development_2022,chinchilla_learning_2023}.
        The non-physical model couples recursively with the physical dyamics.
        This parameterization is expressive, but can be taken to compromise the physicality of the physical model.
        \item 
        Evolve the physical model with physical dynamics, and then sum its predicted output with a modeled output residual \cite{kennedy_bayesian_2001,kaheman_learning_2019,wohlleben_transferability_2023}.
\end{itemize}
Both structures have their merits.
We adopt the second because we favor the interpretation that the physics-based model embedded in a hybrid model should have some standalone usefulness divorced from the residual model.

\subsection{Motivation for PAML}
\label{subsec:motivation}
% PINN
The preponderance of the evidence shows that most real-world systems can be represented to serviceable accuracy by black-box model classes.
Universal approximators such as polynomials, neural networks, and radial basis function networks, and universal dynamic system approximators such as recurrent neural networks and Volterra series are, with sufficient training data, ``prepared to describe virtually any nonlinear dynamics''\cite{sjoberg_nonlinear_1995}.
The question may be, why refer to physics at all?
One answer is that in practice, real-world data is too scarce for learning sufficiently powerful black-box models \cite{wang_learning_2023,levine_framework_2022}.

To engineers, the value proposition of ``physics-informed machine learning'' is that the physical model's predictive performance serves as a constraint or regularization that enhances the machine learning model's generalization capabilities.

% For the scientist, 

% A model is considered physical when it has a high explanatory power proportional to its complexity.
% Explanatory power is a subjective synthesis of the accuracy of the model's physical predictions and the ease with which the model can be deduced from intuition and first principles.
% Two examples follow in which one factor is present to the exclusion of the other.
% On one hand, universal function approximators s, are, in theory, with sufficiently rich training data, 
% ``prepared to describe virtually any nonlinear dynamics''\cite{sjoberg_nonlinear_1995}.
% On the other hand, the natural philosophers of Western Antiquity observed nature and made deeply principled, yet ultimately speculative deductions about its workings.
% % \footnote{cite \emph{History of Animals}, Aristotle or Plato's version}

% The Scientific Revolution defined scientific understanding as principled deduction that leads to testable predictions \cite{hepburn_scientific_2021}.
% One could posit that physics is no more symbolic regression of the universe; and there is no objective standard by which a model can be judged ``physical''; only ``I know it when I see it.''
To scientists, PAML can be a way to supplement the accuracy of a physical model while minimally diminishing its explanatory power.
Explanation is the purpose of physics.
What is and isn't physical is not intrinsic to mathematical laws;
but rather a porduct of modern intellectual history.
Mere centuries ago, the Scientific Revolution \emph{defined} scientific knowledge as the synthesis of logical deduction and experimental verification \cite{hepburn_scientific_2021}.
This is why black-box approximations to natural laws, no matter the accuracy, are not physical: they do not serve the human quest for understanding.

``Physical'' is not synonymous with ``true.''
All modeling ``first principles'' are necessarily perturbative and lose their validity outside of a ``gain-bandwidth lemon'' \cite{astrom_feedback_2008}.
What should be expected of hybrid modeling when the physics is wrong, i.e.~when physics-informed learning turns into physics-\emph{misinformed} learning?
We answer that in such a case, the hybrid model should not conceal broken physics in order to match the data.

% The Scientific Revolution defined scientific understanding as principled deduction that leads to testable predictions \cite{hepburn_scientific_2021}.
% Thus, from a scientific point of view, the value proposition of residual modeling is to increase the predictive accuracy of a physical model while somehow retaining its principledness.
% One way to prove the principledness of a physics augmentation scheme is to demonstrate its falsifiability: if the ``first principles'' are mistaken and the physical model is grossly incompatible with data, the augmented model should be incompatible with data as well.

% The converse of falsification is identification.
% If we are provided with a hybrid model that fits the experimental reality, then we expect that the physical model is correctly identified and makes accurate predictions within its domain of applicability.
% (In section \ref{section:incompatibility}, we translate this abstract desideratum into a mathematical specification.)

\subsection{Principles for PAML}
\label{section:criteria}
% It is not a new idea that adding universal function approximators can make a hybrid model unreliable in practice
% \cite{albers_ensemble_2019,levine_framework_2022,sirlanci_simple_2022,albers_interpretable_2023,wang_learning_2023,wu_learning_2024,zou_hybrid2_2024}.
% Concern has been raised over the risk over overly powerful residual models.
PAML can be viewed as a regularization of black-box modeling that serves scientific integrity and hedges against technical risks such as learning an unstable model \cite{wu_learning_2024}.
We are not the first to raise these concerns and propose heuristic criteria.
One black-box regularization incorporates domain knowledge in the form of a monotonicity constraint \cite{zou_hybrid2_2024}.
APHYNITY interpolates the data with a hybrid model while minimizing the quadratic norm of the residual term \cite{yin_augmenting_2021}.
Another approach leaves the residual unconstrained but imposes a penalty on the physical parameter based on physical prior knowledge \cite{dona_constrained_2022}.

The mathematical content in these works is generally concerned with demonstrating existence or uniqueness criteria for the hybrid model.
While we share the same technical concerns, our focus is on defining qualitative constraints on the hybrid model class that safeguard the physical model's primacy and interpretability.

\subsection{Contributions}
Our zeroth-order incompatibility criterion (Def.~\ref{def:zeroth-order-incompatibility-witness}) requires that the black-box residual model \emph{cannot} approximate the physical model.
This means that if the physical model happens to be true to the data-generating process, the unique best-fitting hybrid model contains a zero residual.
We are not aware of any other hybrid modeling regularization that explicitly boasts such a failure-of-approximation property for the residual model.

Our first-order incompatibility criterion (Def.~\ref{def:first-order-incompatibility-witness}) implies that the unaugmented physical model and the augmented physical model share the same parameters.
This eliminates the need for alternating iterations of re-learning the physical and residual models as in \cite{dona_constrained_2022}.
Instead, one needs only to learn the unaugmented physical model, use the residual to train the residual model, and then the iteration stops.
Furthermore, the hybrid model's physical parameter retains its physical interpretability: under the hood of every best-fitting hybrid model lies a best-fitting physical model.

It is hard to meet these standards using universal function approximators, but with careful designed black-box models it is possible.
We give an example in \S\ref{sec:example} in which the physical model is the linearization of a nonlinear mechanical system.

\section{Notation}
We use script letters to denote classes of signals.
The class of input signals is called \(\mathcal{U}\) and the class of output signals is called \(\mathcal{Y}\).
% Regularity assumptions will be assumed from context.
\begin{align*}
        \mathcal{U} &= \cbr{\text{functions } [0, \infty) \to \mathbb{R}^{\mathsf{d}_u}}
        \\
        \mathcal{Y} &= \cbr{\text{functions } [0, \infty) \to \mathbb{R}^{\mathsf{d}_y}}
\end{align*}
We do not state explicit regularity assumptions when they can be inferred from context.
For model fitting, we shall employ the \(L^2\) norm induced by the inner product:
\begin{align}
        \left\langle u, v\right\rangle &= \lim_{T \to \infty}\frac{1}{T} \int_0^T u(t) v(t) \dif t.
\end{align}
% \subsection{Models}

We also use script letters to denote a model class and indicate parameterization by superscripts.
Of particular importance, we will denote the ``physical'' model class by
\begin{align*}
        \mathcal{H} &= \cbr{\mathcal{H}^\theta: \theta \in \Theta}
        \intertext{and the ``residual'' model class by}
        \mathcal{R} &= \cbr{\mathcal{R}^\gamma: \gamma \in \Gamma}.
\end{align*}
We assume that \(\mathcal{R}\) includes the the zero model defined by \(\mathcal{R}^{\gamma_0}[u] = 0\).

Here, \(\Theta\) and \(\Gamma\) are physical and residual parameters, respectively.
A typical element of \(\mathcal{H}\) is a mapping \(\mathcal{H}^\theta : \mathcal{U} \to \mathcal{Y}\).
For a given \(u\), the application \(y = \mathcal{H}^\theta[u]\) might be given in state-space form as
\begin{align*}
        \dot x(t; \theta) &= f(t, x, u(t); \theta) \\
        y(t; \theta) &= g(t, x, u(t); \theta).
\end{align*}

We denote the hybrid model class as 
\begin{align*}
        \mathcal{H} + \mathcal{R}
        &= \cbr{\mathcal{H}^\theta + \mathcal{R}^\gamma: \theta \in \Theta, \gamma \in \Gamma}
        \\
        \intertext{where the notation \(\mathcal{H}^\theta + \mathcal{R}^\gamma\) signifies that for any \(u \in \mathcal{U}\),}
        (\mathcal{H}^\theta + \mathcal{R}^\gamma)[u]
        &= \mathcal{H}^\theta[u] + \mathcal{R}^\gamma[u].
\end{align*}

\section{Method of incompatibility certificates}
\label{section:incompatibility}
\subsection{Motivation: harmonics of a nonlinear system}
In order to find general principles for how the residual model class \(\mathcal{R}\) should complement the physical model 
class \(\mathcal{Y}\),
we proceed along a Gedankenexperiment where the true system and true residual can be known exactly.

Suppose that the true system \(\mathcal{H}^*\) is a nonlinear single-input single-output model with sufficient stability properties\footnote{e.g.~contraction \cite{lohmiller_contraction_1998}} that periodic inputs result in periodic outputs.
In particular, if \(u(t; \alpha, \omega) = \alpha \sin (\omega t)\), then \(y = \mathcal{H}^*[u]\) can be expanded as\footnote{This kind of harmonic expansion is the starting point of describing function analysis \cite[Chapter 7]{kaheman_learning_2019}, \cite[Chapter 7]{astrom_feedback_2008}. The thought process depicted here is similar to the use of electrochemical impedance spectroscopy to assess battery health parameters.}
\begin{align*}
        y(t; \alpha, \omega) &= \sum_{k = 0}^{\infty} M_k(\alpha, \omega) \sin(k\omega t + \phi_n(\alpha)).
        \intertext{It is common for the ``physical'' model class to be linear. In this model class, we prescribe a transfer function \((\hat M(\omega; \theta), \hat \phi(\omega; \theta))\) by some LTI system parameterization \(\theta\) such as controllable canonical state space form.}
        \hat y(t; \alpha; \omega) &= \hat M(\omega) \sin (k\omega t + \hat\phi(\omega)).
        \intertext{This is justified on the basis of small-signal linearization, if it is conjectured to hold that \(\mathcal{H}^*\) is Fr\'echet differentiatble at 0 in the \(u\) direction:}
        \hat y(t; \alpha; \omega)
        &\approx \lim_{\alpha \to 0} \alpha^{-1} y(t; \alpha, \omega).
        \intertext{Hypothesizing that \(M_1(\alpha, \omega)\) does not depend on \(\alpha\) so that the \(k=1\) term is perfectly captured by \(\hat y\), we get the true residual}
        r(t; \alpha, \omega) &= 
        \sum_{k = 0, k\neq 1}^{\infty} M_n(\alpha, \omega) \sin(k\omega t + \phi_n(\alpha)),
\end{align*}
a superposition of sinusoids at frequencies different from \(\omega\).
We observe:
\begin{itemize}
        \item Because pure sinusoids of different frequencies are orthogonal (in an inner product made precise in the next subsection), we have the guarantee that \(r\) and \(\hat y\) are orthogonal.
        It guarantees \emph{a fortiori} that the residual model cannot match the physical model.
        \item Direct differentiation reveals that the sensitivity \(\partial_\theta \hat y\) is a pure sinusoid at frequency \(\omega\).
        This means that the residual model does not interact with the first-order optimality conditions for fitting \(\hat y\) by least squares.
\end{itemize}

We codify these two observations in Definitions~\ref{def:zeroth-order-incompatibility-witness} and \ref{def:first-order-incompatibility-witness}.

% The following definition is one way to formalize the subjective notion that the residuals are qualitatively different from the physical model.\footnote{introduce notation earlier and better?}

\subsection{Definitions}
The following definitions relate the model \emph{classes} \(\mathcal{H}\) and \(\mathcal{R}\), which means that we make universal quantifications of the form ``for all \(\theta \in \Theta\) and \(\gamma \in \Gamma\), \(\mathcal{H}^\theta\) and \(\mathcal{R}^\gamma\) \ldots''
As a consequence, the takeaways are insensitive to any particular (mis-)identification of \(\theta\) or \(\gamma\) from data.

The term \emph{witness}, borrowed from theoretical computer science, refer to a test case that exemplifies the desired property.
We use it to refer to an input sequence that satisfies a qualitative property for the entire hybrid model class \(\mathcal{H} + \mathcal{R}\).
The definitions and their consequences are heuristically compelling if there are as many incompatibility witnesses as possible.

\begin{definition}[Zeroth-order incompatibility witness]
        \label{def:zeroth-order-incompatibility-witness}
        We say that \(u \in \mathcal{U}\) is a \emph{zeroth-order incompatibility witness} for the hybrid model class \(\mathcal{H} + \mathcal{R}\) if for all \(\theta \in \Theta\) and \(\gamma \in \Gamma\),
        \begin{align*}
                \lim_{T \to \infty}
                \frac{1}{T}\int_0^T 
                \mathcal{R}^\gamma[u](t) \mathcal{H}^\theta[u](t) 
                 \dif t = 0.
        \end{align*}
        We write this as \(\mathcal{H} \perp_u \mathcal{R}\).
\end{definition}

\begin{definition}[First-order incompatibility witness]
        \label{def:first-order-incompatibility-witness}
        We say that \(u \in \mathcal{U}\) is a \emph{first-order incompatibility witness} for the hybrid model class \(\mathcal{H} + \mathcal{R}\) if for all \(\theta \in \Theta\) and \(\gamma \in \Gamma\),
        \begin{align*}
                \lim_{T \to \infty}
                \frac{1}{T}\int_0^T 
                 \mathcal{R}^\gamma[u](t) \dpd{}{\theta} \mathcal{H}^\theta[u](t)
                 \dif t = 0.
        \end{align*}
        We write this as \(T\mathcal{H} \perp_u \mathcal{R}\).
\end{definition}

\subsection{Consequences}
In this discussion, let \(y \in \mathcal{Y}\) be an observed signal.
We assume that \(\mathcal{H}\) is correctly specified; i.e.~\(y = \mathcal{H}^{\theta_0} u\) for some \(\theta_0 \in \Theta\), \(u \in \mathcal{U}\).
The zeroth-order incompatibility property ensures that the residual model cannot overcome parametric misidentification of \(\theta\).
Logically, it allows us to generalize existential quantifications (\textbf{there exists}) that hold at a single parameter value to universal quantifications (\textbf{for all}) that hold for all possible parameter values.

\begin{fact}[Zeroth-order identification]\footnote{This and the following ``obvious fact'' hold regardless of incompatibility properties.}
        \textbf{There exists} \(\hat\theta \in {\Theta}\) such that
        \begin{gather*}
                (\hat \theta, \hat \gamma)
                \in \argmin_{(\theta, \gamma) \in \Theta \times \Gamma} 
                \left\|y - \del{\mathcal{H}^{\theta}  + \mathcal{R}^{\gamma}}u\right\|
                \\
                \implies 
                \mathcal{R}^{\hat \gamma} u = 0.
        \end{gather*}
\end{fact}
\begin{proof}
        Take \(\hat\theta = \theta_0\).
\end{proof}

\begin{theorem}[Zeroth-order incompatibility]
        Suppose furthermore
        that \(\mathcal{H} \perp_u \mathcal{R}\).
        Then \textbf{for all} \(\hat\theta \in {\Theta}\),
        \begin{gather*}
                (\hat \theta, \hat \gamma)
                \in \argmin_{(\theta, \gamma) \in \Theta \times \Gamma} 
                \left\|y - \del{\mathcal{H}^{\theta}  + \mathcal{R}^{\gamma}}u\right\|
                \\
                \implies 
                \mathcal{R}^{\hat \gamma} u = 0.
        \end{gather*}
\end{theorem}
\begin{proof}
        Let us expand the objective function as
        \begin{align}
                \ell(\theta, \gamma) 
                        &= \left\|y - \del{\mathcal{H}^{\theta}  + \mathcal{R}^{\gamma}}u\right\|^2
                \\
                \begin{split}
                        &= \left\|y - \mathcal{H}^\theta u\right\|^2 
                        - 2 \left\langle y, \mathcal{R}^\gamma u \right\rangle
                        \\
                        &\quad
                        - 2 \left\langle \mathcal{H}^\theta u, \mathcal{R}^\gamma u \right\rangle
                        + \left\| \mathcal{R}^\gamma u\right\|^2
                \end{split}
                \\
                \intertext{By the zeroth-order incompatibility hypothesis,}
                &= \left\|y - \mathcal{H}^\theta u\right\|^2 
                        + \left\| \mathcal{R}^\gamma u\right\|^2,
        \end{align}
        which is minimized at \(\theta = \theta_0\) and \(\mathcal{R}^\gamma u = 0\).
\end{proof}

The first-order incompatibility property means that there is no need to iterate between fitting \(\hat\theta\) and \(\hat\gamma\); it is possible to estimate \(\hat \theta\) and then \(\hat \gamma\).
We make claims about local minima. If \(\ell\) happens to be convex, then local minimizers of \(\ell\) are also pseudo-true parameter values.
The physical and hybrid models have the same pseudo-true parameters.
\begin{fact}[First-order identification]
        \textbf{There exists} \(\hat \gamma \in {\Gamma}\) such that
        \begin{gather*}
                \hat \theta \operatorname*{\ locally\ minimizes\ }_{\theta \in \Theta} \left\|y - \del{\mathcal{H}^{\hat \theta}  + \mathcal{R}^{\hat \gamma}}u\right\|^2
                \\
                \implies
                \hat \theta \operatorname*{\ locally\ minimizes\ }_{\theta \in \Theta} \left\|y - \del{\mathcal{H}^{\hat \theta} }u\right\|^2.
        \end{gather*}  
\end{fact}
\begin{proof}
        Take \(\gamma = \gamma_0\), the zero residual.
\end{proof}
\begin{theorem}[First-order incompatibility]
        \label{thm:first-order-incompatibility}
        Suppose furthermore that \(T\mathcal{H} \perp_u \mathcal{R}\)
        Then \textbf{for all} \(\hat \gamma \in {\Gamma}\),
        \begin{gather*}
                \hat \theta \operatorname*{\ locally\ minimizes\ }_{\theta \in \Theta} \left\|y - \del{\mathcal{H}^{\hat \theta}  + \mathcal{R}^{\hat \gamma}}u\right\|^2
                \\
                \implies
                \hat \theta \operatorname*{\ locally\ minimizes\ }_{\theta \in \Theta} \left\|y - \del{\mathcal{H}^{\hat \theta} }u\right\|^2.
        \end{gather*}  
\end{theorem}
\begin{proof}
        The first-order optimality conditions for least squares states that for any operator \(p\) representing a directional derivative\footnote{In the terminology of smooth manifolds, a tangent vector, hence the notation \(T\mathcal{H} \perp_u \mathcal{R}\).} with respect to \(\theta\),
        \begin{align}
                0 = p\ell(\theta, \gamma) 
                        &= p\left\|y - \del{\mathcal{H}^{\theta}  + \mathcal{R}^{\gamma}}u\right\|^2
                \\
                \begin{split}
                        &= p\left\|y - \mathcal{H}^\theta u\right\|^2 
                        - 2p \left\langle y, \mathcal{R}^\gamma u \right\rangle
                        \\
                        &\quad
                        - 2p \left\langle \mathcal{H}^\theta u, \mathcal{R}^\gamma u \right\rangle
                        + p\left\| \mathcal{R}^\gamma u\right\|^2
                \end{split}
                \\
                \intertext{By the product rule and first-order incompatibility hypothesis, \(p \left\langle \mathcal{H}^\theta u, \mathcal{R}^\gamma u \right\rangle = \left\langle p\mathcal{H}^\theta u, \mathcal{R}^\gamma u \right\rangle = 0\), hence}
                0 &= p\left\|y - \mathcal{H}^\theta u\right\|^2,
        \end{align}
        which is the first-order optimality condition for \(\theta\) to minimize \(\left\|y - \mathcal{H}^\theta u\right\|^2\).
\end{proof}

\section{Example}
\label{sec:example}
We will give an example of constructing a residual model with the requisite orthogonality properties to the following linearized physical model for a mechanical pendulum:
%  \(\hat y = \mathcal{H}^\theta u\), \(\theta=\zeta\) will be the small-signal linearization around the lower equilibrium:
\begin{subequations}
\label{eq:physical-model}
\begin{align}
        \ddot {\hat x}(t) + 2 \zeta \omega_0 \dot {\hat x}(t) + \omega_0^2 \hat x(t) &= u(t), \quad 0 \leq t \leq T
        \\
        \hat y(t) &= \hat x(t)
\end{align}
\end{subequations}
where \(x\) is the angle from the lower equilibrium, the input \(u\) represents torque, and the output \(y\) is the horizontal position of the pendulum end.
The parameters are \(\zeta\), the damping ratio, and \(\omega_0\), the small-signal resonant frequency.
We assume that \(\omega_0\) is known (after all, it can be calculated from the mechanical properties of the system), but \(\zeta\) is unknown pending estimation.

The true system is nonlinear:\footnote{A common benchmark in hybrid modeling \cite{yin_augmenting_2021,dona_constrained_2022}.}
\begin{subequations}
\label{eq:real-system} 
\begin{align}
        \ddot x(t) + 2 \zeta \omega_0 \dot x(t) + \omega_0^2 \sin(x(t)) &= u(t) 
        \\ 
        y(t) &= \sin (x(t)) 
\end{align}         
\end{subequations}

\subsection{Residual model specification}
We design a residual model in Hammerstein form (static nonlinearity followed by a linear system) informed by the discussion in \S\ref{subsec:motivation} about harmonic generation.
Fix a characteristic amplitude \(\alpha > 0\) on which the residual model class and incompatibility model set depend.
% The input class will be \(\mathcal{U}_\alpha\) consisting of signals whose values are bounded in \([-\alpha, \alpha]\).
Define the residual model  \(r = \mathcal{R}_\alpha^\gamma u\), \(\gamma = (A, B) \in \mathbb{R}^{m} \times \mathbb{R}^{m \times n}\) by
\begin{align*}
        \sum_{d = 1}^m A_d \dod[d]{}{t} r(t)
        &= \sum_{d = 1}^m \sum_{k = 1}^{n} B_{d, k} \dod[d]{}{t} \sbr{T_{2k + 1}(u(t)/\alpha)}
        % a_3 \dod[3]{r}{t}  &= \sum_{k = 1}^{n} b_k T_k(u(t))
        % \\
        % e(t) &= c_2 \dod[2]{r}{t} + c_1 r + c_0 r + \sum_{k = 1}^{n} d_k T_k(u(t))
\end{align*}
where \(T_j\) is the Chebyshev polynomial of degree \(j\).
Because the mechanical system can be known physically to have mirror symmetry and the physical model is odd (negative \(u\) results in negative \(y\)),
we have constructed this residual model to be odd as well.

Define the witness class \(\tilde {\mathcal{U}}_\alpha\) by
\begin{align*}
        \tilde {\mathcal{U}} = \cbr{
                \alpha \cos(\omega t + \phi) : \omega > 0, \phi \in [0, 2\pi)
        }.
\end{align*}
For this construction, \(\mathcal{H}_\alpha\) and \(\mathcal{R}_\alpha\) need to have the same \(\alpha\).

\begin{proposition}
        Fix \(\alpha > 0\). For all \(u \in \tilde {\mathcal{U}}_\alpha\), \(\mathcal{H} \perp_u \mathcal{R}_\alpha\) and \(T\mathcal{H} \perp_u \mathcal{R}_\alpha\).
\end{proposition}
\begin{proof}
        Let \(u(t) =  \alpha \cos(\omega t + \phi)\).
        For \(\hat y = \mathcal{H}^\theta_\alpha u\) and \(r = \mathcal{R}_\alpha^\gamma u\),
        we claim there exist values of the subscripted constants \(\alpha\) and \(\beta\) such that the following three identities hold.
        By linear systems theory,
        \begin{subequations}
        \begin{align}
        \label{eq:physical-model-solution}
                \hat y(t) &= \alpha_1 \cos(\omega t + \phi_1).
                \intertext{By differentiating \eqref{eq:physical-model}, we have}
                \dpd{}{\theta} \hat y(t) &= \alpha_p \cos(\omega t + \phi_p).\label{eq:physical-model-derivative}
                \intertext{By the property \(T_\ell \cos(\theta) = \cos(\ell\theta)\) of the Chebyshev polynomials,}
                r(t) &=
                \sum_{k = 1}^m \alpha_{2k + 1} \cos((2k + 1) \omega + \phi_{2k + 1}).\label{eq:residual-model-solution}
        \end{align}
        \end{subequations}

        Finally, for all \(\omega' \neq \omega\) and \(\phi', \psi'\),
        \begin{align}\label{eq:cosine-orthogonality}
                \lim_{T \to \infty} \frac{1}{T} \int_0^T \cos(\omega t + \phi') \cos(\omega' t + \psi') \dif t
                 = 0.
        \end{align}

        It follows from \eqref{eq:physical-model-solution}, \eqref{eq:residual-model-solution}, and \eqref{eq:cosine-orthogonality} that \(\mathcal{H} \perp_u \mathcal{R}_\alpha\).
        It follows from \eqref{eq:physical-model-derivative}, \eqref{eq:residual-model-solution}, and \eqref{eq:cosine-orthogonality} that \(T\mathcal{H} \perp_u \mathcal{R}_\alpha\).
\end{proof}

\subsection{Numerical results}
\begin{table}
\centering
\begin{tabular}{llr}
        Constant & Meaning & Value
        \\\hline
        \(\zeta\) & small-signal damping ratio & \(0.5\)
        \\
        \(\omega_0\) & small-signal resonant frequency & \(2.0\)
        \\
        \(T\) & experiment end time & \(80\)
        \\
        \(\omega\) & test frequency for \(\zeta\) estimation & \(1\)
        \\
        \hline
        \(d\) & LTI order of residual model & 5 \\
        \(n\) & number of Chebyshev polynomials & 6
        \\
\end{tabular}
\caption{\label{table:constants}}
\end{table}
In this section, we demonstrate two instances of the end-to-end hybrid modeling workflow.
The ``medium-signal'' example uses an input amplitude \(\alpha = 4\), which is large enough to excite dynamic and output nonlinearities in the real system.
The ``large-signal'' example uses an input amplitude \(\alpha = 10\), which is enough to excite qualitatively different behavior from the linearized ``physical'' model.
``Mismatch norm'' refers to the \(L^2\) norm (root mean square) of the difference of two signals, normalized by the \(L^2\) norm of the real system output \(y\).
\begin{enumerate}
        \item \textbf{Fit the physical model.} Apply a known sinusoidal input and the Prediction Error Method to estimate \(\zeta\) in \eqref{eq:physical-model}. Call this estimate \(\hat \zeta\).
        \begin{itemize}
                \item medium-signal: Fig.~\ref{figure:physical-id-alpha-med}, large-signal: Fig.~\ref{figure:physical-id-alpha-10.0}
        \end{itemize}
        \item \textbf{Fit the residual model.} Apply a pseudorandom input to obtain an empirical residual between the eal system and the physical model. Fit the LTI parameters of the residual model by the Prediction Error Method.
        \begin{itemize}
                \item medium-signal: Fig.~\ref{figure:residual-id-alpha-med}, large-signal: Fig.~\ref{figure:residual-id-alpha-10.0}
        \end{itemize}
        \item \textbf{Validate the hybrid model.} Apply a different pseudorandom input sequence and compare the hybrid model's predictions to the real system's output.
        % \footnote{It is realistic that the inputs of practical importance differ from those used in estimation. \color{red} talk about EIS and other battery testing inputs vs UDDS\ldots papers to cite?}
        \begin{itemize}
                \item medium-signal: Fig.~\ref{figure:validation-alpha-med}, large-signal: Fig.~\ref{figure:validation-alpha-10.0}
        \end{itemize}
\end{enumerate} 

In the medium-signal regime, the hybrid model reduces the mismatch norm of the physical model by a factor of roughly two in both training (from 12.08\% to 5.24\%) and validation (from 12.45\% to 6.25\%).
There is a convincing fit between the actual residual \(y - \hat y\) and the predicted residual \(r\) on pseudorandom validation data (Fig.~\ref{figure:validation-alpha-med}).
This suggests that odd harmonic content arising both from the nonlinear dynamics and the output nonlinearity are faithfully captured by the residual model's Hammerstein form.

To verify first-order incompatibility (Thm.~\ref{thm:first-order-incompatibility}) we used the Prediction Error Method to re-estimate \(\hat \zeta\) on the residual-compensated signal \(y - r\).
The iterated estimate differs by -0.00030 or -0.047\%,
theoretically be zero for \(T \to \infty\),
but close enough to be regarded as practically convergent.
One-step convergence is an improvement upon
alternating physical and residual estimates, which are not even certain to converge.

In the large-signal regime, we note that the both the physical and the hybrid model have mismatch close to 90\% in both training and validation.
We view this failure of approximation as a good thing and attribute it to the orthogonality design of \(\mathcal{H}_\alpha\) and \(\mathcal{R}_\alpha\).
The hybrid model has been effectively regularized \emph{against} fitting the real signal when the physical model is inadequate.
% Moreover, the hybrid model has poor generalization
% Physically speaking, the small-signal linearization is inadequate to describe dynamic behavior where the pendulum swings up and over its highest point.

There is a mechanical explanation for why the linear model is ineffective for large values of \(\alpha\).
Physically speaking, the pendulum swings up and over its highest point, generating qualitatively different dynamics than the linearization can predict.
The real system \ref{eq:real-system} has an equilibrium for constant \(u(t) = u_0\) only in the range \(u_0 \in [-\omega_0^2, \omega_0^2]\).
On the other hand, the linearized system \ref{eq:physical-model} has the same forced equilibria with no qualification on \(u_0\).

This ablation study uncovers how our physics-informed learning system behaves when it does not have the correct physics information.
We confront the potentially uncomfortable fact that \eqref{eq:physical-model} is not the right physics for physics-informed learning.
If \(\mathcal{R}\) were a universal function approximation class, it would be possible to get an immaculate fit \(y \approx \hat y + r\) in training and validation.
But to do so would be an exercise in physics-\emph{misinformed} learning.

% \begin{rev}
%         1st-order result 0.63477147
% \end{rev}

Constants associated with the physical and residual models can be found in in Table~\ref{table:constants}.
Estimation methods and implementation details can be found in Appendix~\ref{section:methods}.

\begin{figure}
        \includegraphics[width=\linewidth]{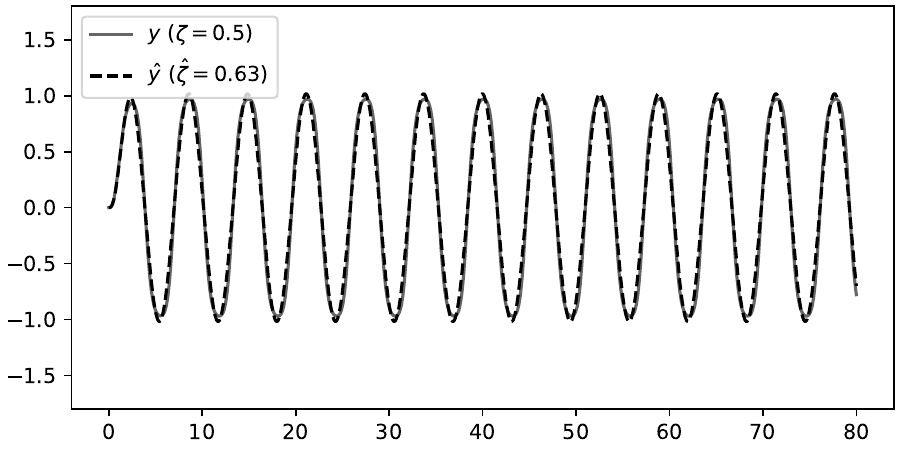}
        \caption{\label{figure:physical-id-alpha-med} Real trajectory and physical model identification at \(\alpha = 4\).
        Mismatch norm is 12.01\%.}
\end{figure}

\begin{figure}
        \includegraphics[width=\linewidth]{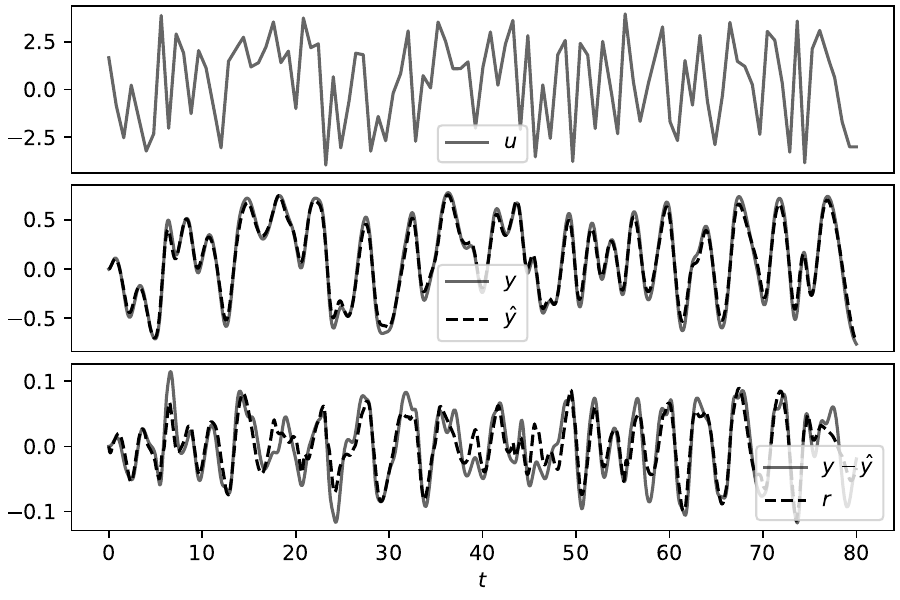}
        \caption{\label{figure:residual-id-alpha-med} Real trajectory and residual model identification at \(\alpha = 4\).
        Physical model mismatch norm is 12.08\%; hybrid model mismatch norm is 5.24\%.}
\end{figure}

\begin{figure}
        \includegraphics[width=\linewidth]{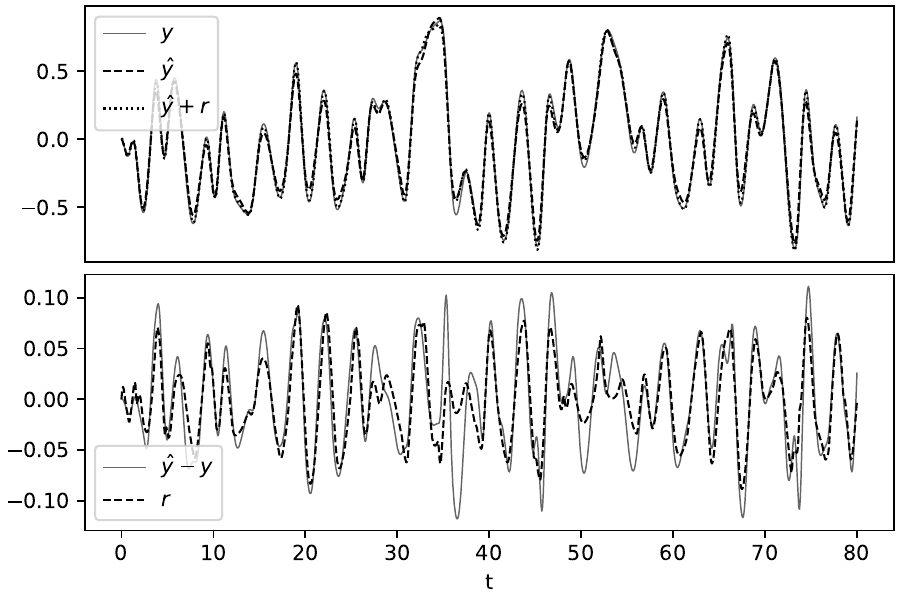}
        \caption{\label{figure:validation-alpha-med} Real trajectory and hybrid model on a pseudorandom validation input at \(\alpha = 4\).
        Physical model mismatch norm is 12.45\%; hybrid model mismatch norm is 6.25\%.}
\end{figure}

% \subsection{Large-signal}
% \label{section:large-signal}
% hybrid modeling fails, which correctly reflects the physical model's severe misspecification
\begin{figure}
        \includegraphics[width=\linewidth]{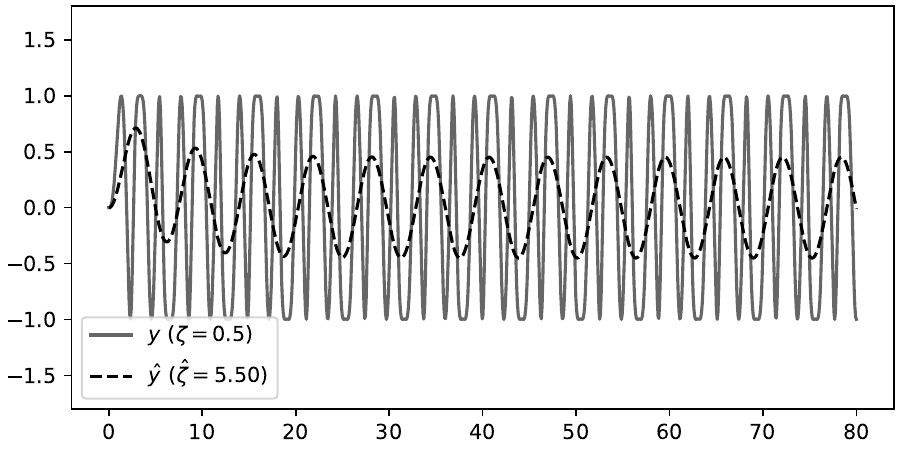}
        \caption{\label{figure:physical-id-alpha-10.0} Real trajectory and physical model identification at \(\alpha = 10\).
        Mismatch norm is 90.71\%.}
\end{figure}

\begin{figure}
        \includegraphics[width=\linewidth]{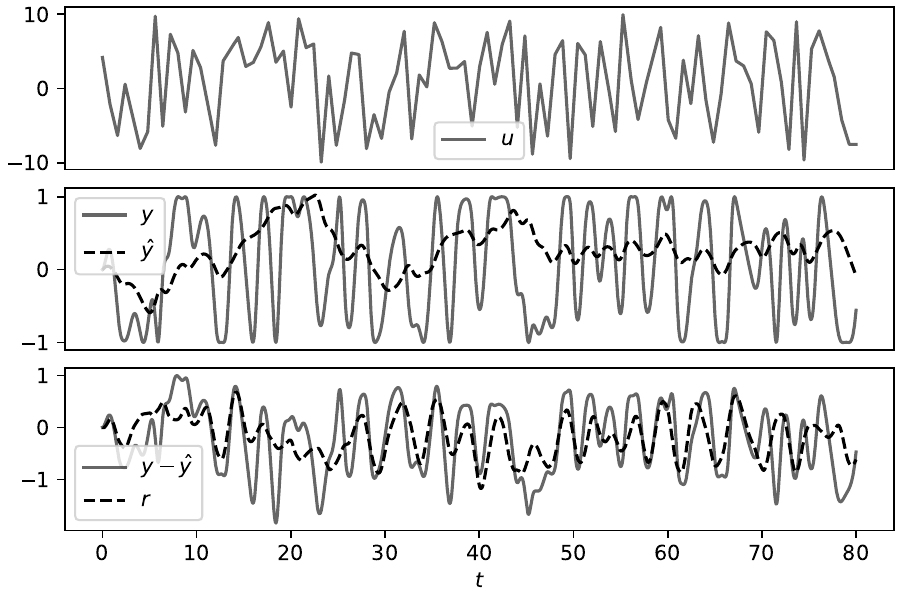}
        \caption{\label{figure:residual-id-alpha-10.0} Real trajectory and residual model identification at \(\alpha = 10\).
        Physical model mismatch norm is 99.84\%; hybrid model mismatch norm is 76.42\%.}
\end{figure}

\begin{figure}
        \includegraphics[width=\linewidth]{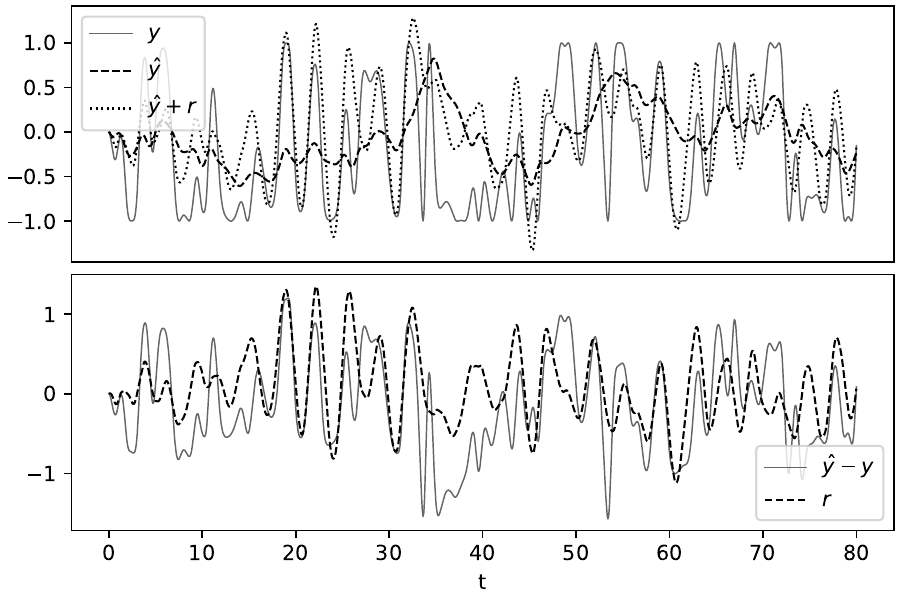}
        \caption{\label{figure:validation-alpha-10.0} Real trajectory and hybrid model on a pseudorandom validation input at \(\alpha = 10\).
        Physical model mismatch norm is 88.76\%; hybrid model mismatch norm is 81.83\%.}
\end{figure}

\section{Conclusion}
% This paper proposes two qualitative criteria that improve the trustworthiness and attribution of physics-augmented learning.
Zeroth-order incompatibility theoretically guarantees that a hybrid model is \emph{not} robust to physics mis-information.
First-order incompatibility guarantees one-step convergence of the hybrid model.
These criteria certify the trustworthiness and attribution of physics-augmented learning.

\addtolength{\textheight}{0cm}   % This command serves to balance the column lengths

%%%%%%%%%%%%%%%%%%%%%%%%%%%%%%%%%%%%%%%%%%%%%%%%%%%%%%%%%%%%%%%%%%%%%%%%%%%%%%%%

%%%%%%%%%%%%%%%%%%%%%%%%%%%%%%%%%%%%%%%%%%%%%%%%%%%%%%%%%%%%%%%%%%%%%%%%%%%%%%%%

%%%%%%%%%%%%%%%%%%%%%%%%%%%%%%%%%%%%%%%%%%%%%%%%%%%%%%%%%%%%%%%%%%%%%%%%%%%%%%%%
\section*{APPENDIX}

\section{Methods}
\label{section:methods}
% \subsection{Estimation}
We generated pseudorandom inputs by linearly interpolating independent \(\text{Uniform}(-\alpha, \alpha)\) random variables at a period of 0.8.
We implemented the Prediction Error Method for \(\hat\zeta\) by using a bisection-based scalar minimization routine initialized using the modulating function method \cite{niethammer_parameter_2001}.
We integrated the ordinary differential equation using Diffrax\cite{kidger_neural_2022} with a 5th order implicit solver.
We implemented the Prediction Error Method for the residual using a gradient-based optimizer.

% \section*{ACKNOWLEDGMENT}

% %%%%%%%%%%%%%%%%%%%%%%%%%%%%%%%%%%%%%%%%%%%%%%%%%%%%%%%%%%%%%%%%%%%%%%%%%%%%%%%%

% References are important to the reader; therefore, each citation must be complete and correct. If at all possible, references should be commonly available publications.

\bibliography{IEEEabrv,export}

\begin{thebibliography}{10}
\providecommand{\url}[1]{#1}
\csname url@rmstyle\endcsname
\providecommand{\newblock}{\relax}
\providecommand{\bibinfo}[2]{#2}
\providecommand\BIBentrySTDinterwordspacing{\spaceskip=0pt\relax}
\providecommand\BIBentryALTinterwordstretchfactor{4}
\providecommand\BIBentryALTinterwordspacing{\spaceskip=\fontdimen2\font plus
\BIBentryALTinterwordstretchfactor\fontdimen3\font minus \fontdimen4\font\relax}
\providecommand\BIBforeignlanguage[2]{{%
\expandafter\ifx\csname l@#1\endcsname\relax
\typeout{** WARNING: IEEEtran.bst: No hyphenation pattern has been}%
\typeout{** loaded for the language `#1'. Using the pattern for}%
\typeout{** the default language instead.}%
\else
\language=\csname l@#1\endcsname
\fi
#2}}

\bibitem{chinchilla_learning_2023}
\BIBentryALTinterwordspacing
R.~Chinchilla, V.~M. Deshpande, A.~Chakrabarty, and C.~R. Laughman, ``Learning {Residual} {Dynamics} via {Physics}-{Augmented} {Neural} {Networks}: {Application} to {Vapor} {Compression} {Cycles},'' in \emph{2023 {American} {Control} {Conference} ({ACC})}, May 2023, pp. 4069--4076, iSSN: 2378-5861.
\BIBentrySTDinterwordspacing

\bibitem{wohlleben_development_2022}
M.~Wohlleben, A.~Bender, S.~Peitz, and W.~Sextro, ``\BIBforeignlanguage{en}{Development of a {Hybrid} {Modeling} {Methodology} for {Oscillating} {Systems} with {Friction}},'' in \emph{\BIBforeignlanguage{en}{Machine {Learning}, {Optimization}, and {Data} {Science}}}, G.~Nicosia, V.~Ojha, E.~La~Malfa, G.~La~Malfa, G.~Jansen, P.~M. Pardalos, G.~Giuffrida, and R.~Umeton, Eds.\hskip 1em plus 0.5em minus 0.4em\relax Cham: Springer International Publishing, 2022, pp. 101--115.

\bibitem{kaheman_learning_2019}
\BIBentryALTinterwordspacing
K.~Kaheman, E.~Kaiser, B.~Strom, J.~N. Kutz, and S.~L. Brunton, ``\BIBforeignlanguage{en}{Learning {Discrepancy} {Models} {From} {Experimental} {Data}},'' Sept. 2019, arXiv:1909.08574 [cs, eess, stat].
\BIBentrySTDinterwordspacing

\bibitem{ahmed_closures_2021}
\BIBentryALTinterwordspacing
S.~E. Ahmed, S.~Pawar, O.~San, A.~Rasheed, T.~Iliescu, and B.~R. Noack, ``On closures for reduced order models—{A} spectrum of first-principle to machine-learned avenues,'' \emph{Physics of Fluids}, vol.~33, no.~9, p. 091301, Sept. 2021.
\BIBentrySTDinterwordspacing

\bibitem{kennedy_bayesian_2001}
\BIBentryALTinterwordspacing
M.~C. Kennedy and A.~O'Hagan, ``\BIBforeignlanguage{en}{Bayesian calibration of computer models},'' \emph{\BIBforeignlanguage{en}{Journal of the Royal Statistical Society: Series B (Statistical Methodology)}}, vol.~63, no.~3, pp. 425--464, 2001, \_eprint: https://onlinelibrary.wiley.com/doi/pdf/10.1111/1467-9868.00294.
\BIBentrySTDinterwordspacing

\bibitem{saveriano_data-efficient_2017}
\BIBentryALTinterwordspacing
M.~Saveriano, Y.~Yin, P.~Falco, and D.~Lee, ``Data-efficient control policy search using residual dynamics learning,'' in \emph{2017 {IEEE}/{RSJ} {International} {Conference} on {Intelligent} {Robots} and {Systems} ({IROS})}, Sept. 2017, pp. 4709--4715, iSSN: 2153-0866.
\BIBentrySTDinterwordspacing

\bibitem{rackauckas_universal_2021}
\BIBentryALTinterwordspacing
C.~Rackauckas, Y.~Ma, J.~Martensen, C.~Warner, K.~Zubov, R.~Supekar, D.~Skinner, A.~Ramadhan, and A.~Edelman, ``\BIBforeignlanguage{en}{Universal {Differential} {Equations} for {Scientific} {Machine} {Learning}},'' Nov. 2021, arXiv:2001.04385 [cs, math, q-bio, stat].
\BIBentrySTDinterwordspacing

\bibitem{sjoberg_nonlinear_1995}
\BIBentryALTinterwordspacing
J.~Sjöberg, Q.~Zhang, L.~Ljung, A.~Benveniste, B.~Delyon, P.-Y. Glorennec, H.~Hjalmarsson, and A.~Juditsky, ``\BIBforeignlanguage{en}{Nonlinear black-box modeling in system identification: a unified overview},'' \emph{\BIBforeignlanguage{en}{Automatica}}, vol.~31, no.~12, pp. 1691--1724, Dec. 1995.
\BIBentrySTDinterwordspacing

\bibitem{peli_physics-based_2022}
\BIBentryALTinterwordspacing
R.~Peli, A.~Menafoglio, M.~Cervino, L.~Dovera, and P.~Secchi, ``\BIBforeignlanguage{en}{Physics-based {Residual} {Kriging} for dynamically evolving functional random fields},'' \emph{\BIBforeignlanguage{en}{Stochastic Environmental Research and Risk Assessment}}, vol.~36, no.~10, pp. 3063--3080, Oct. 2022.
\BIBentrySTDinterwordspacing

\bibitem{ebers_discrepancy_2023}
\BIBentryALTinterwordspacing
M.~R. Ebers, K.~M. Steele, and J.~N. Kutz, ``Discrepancy {Modeling} {Framework}: {Learning} missing physics, modeling systematic residuals, and disambiguating between deterministic and random effects,'' Nov. 2023, arXiv:2203.05164 [math, stat].
\BIBentrySTDinterwordspacing

\bibitem{wohlleben_transferability_2023}
\BIBentryALTinterwordspacing
M.~Wohlleben, L.~Muth, S.~Peitz, and W.~Sextro, ``\BIBforeignlanguage{en}{Transferability of a discrepancy model for the dynamics of electromagnetic oscillating circuits},'' \emph{\BIBforeignlanguage{en}{PAMM}}, vol.~23, no.~2, p. e202300039, 2023, \_eprint: https://onlinelibrary.wiley.com/doi/pdf/10.1002/pamm.202300039.
\BIBentrySTDinterwordspacing

\bibitem{shi_neural_2019}
\BIBentryALTinterwordspacing
G.~Shi, X.~Shi, M.~O’Connell, R.~Yu, K.~Azizzadenesheli, A.~Anandkumar, Y.~Yue, and S.-J. Chung, ``Neural {Lander}: {Stable} {Drone} {Landing} {Control} {Using} {Learned} {Dynamics},'' in \emph{2019 {International} {Conference} on {Robotics} and {Automation} ({ICRA})}, May 2019, pp. 9784--9790, iSSN: 2577-087X.
\BIBentrySTDinterwordspacing

\bibitem{belbute-peres_combining_2020}
\BIBentryALTinterwordspacing
F.~d.~A. Belbute-Peres, T.~D. Economon, and J.~Z. Kolter, ``Combining {Differentiable} {PDE} {Solvers} and {Graph} {Neural} {Networks} for {Fluid} {Flow} {Prediction},'' Aug. 2020, arXiv:2007.04439 [physics, stat].
\BIBentrySTDinterwordspacing

\bibitem{yin_augmenting_2021}
\BIBentryALTinterwordspacing
Y.~Yin, V.~L. Guen, J.~Dona, E.~d. Bézenac, I.~Ayed, N.~Thome, and P.~Gallinari, ``\BIBforeignlanguage{en}{Augmenting physical models with deep networks for complex dynamics forecasting*},'' \emph{\BIBforeignlanguage{en}{Journal of Statistical Mechanics: Theory and Experiment}}, vol. 2021, no.~12, p. 124012, Dec. 2021, publisher: IOP Publishing and SISSA.
\BIBentrySTDinterwordspacing

\bibitem{dona_constrained_2022}
\BIBentryALTinterwordspacing
J.~Donà, M.~Déchelle, M.~Lévy, and P.~Gallinari, ``Constrained physical-statistics models for dynamical system identification and prediction,'' in \emph{{ICLR} 2022 - {The} {Tenth} {International} {Conference} on {Learning} {Representations}}, virtual event, France, Apr. 2022.
\BIBentrySTDinterwordspacing

\bibitem{levine_framework_2022}
\BIBentryALTinterwordspacing
M.~Levine and A.~Stuart, ``\BIBforeignlanguage{en}{A framework for machine learning of model error in dynamical systems},'' \emph{\BIBforeignlanguage{en}{Communications of the American Mathematical Society}}, vol.~2, no.~07, pp. 283--344, 2022.
\BIBentrySTDinterwordspacing

\bibitem{manzi_discovering_2020}
\BIBentryALTinterwordspacing
M.~Manzi and M.~Vasile, ``Discovering {Unmodeled} {Components} in {Astrodynamics} with {Symbolic} {Regression},'' in \emph{2020 {IEEE} {Congress} on {Evolutionary} {Computation} ({CEC})}, July 2020, pp. 1--7.
\BIBentrySTDinterwordspacing

\bibitem{wang_learning_2023}
\BIBentryALTinterwordspacing
K.~A. Wang, M.~E. Levine, J.~Shi, and E.~B. Fox, ``Learning {Absorption} {Rates} in {Glucose}-{Insulin} {Dynamics} from {Meal} {Covariates},'' Apr. 2023, arXiv:2304.14300 [cs, math, q-bio].
\BIBentrySTDinterwordspacing

\bibitem{hepburn_scientific_2021}
\BIBentryALTinterwordspacing
B.~Hepburn and H.~Andersen, ``Scientific {Method},'' in \emph{The {Stanford} {Encyclopedia} of {Philosophy}}, summer 2021~ed., E.~N. Zalta, Ed.\hskip 1em plus 0.5em minus 0.4em\relax Metaphysics Research Lab, Stanford University, 2021.
\BIBentrySTDinterwordspacing

\bibitem{astrom_feedback_2008}
K.~J. Åström and R.~M. Murray, \emph{Feedback systems: an introduction for scientists and engineers}.\hskip 1em plus 0.5em minus 0.4em\relax Princeton: Princeton University Press, 2008, oCLC: ocn183179623.

\bibitem{wu_learning_2024}
\BIBentryALTinterwordspacing
J.-L. Wu, M.~E. Levine, T.~Schneider, and A.~Stuart, ``\BIBforeignlanguage{en}{Learning {About} {Structural} {Errors} in {Models} of {Complex} {Dynamical} {Systems}},'' May 2024, arXiv:2401.00035 [physics].
\BIBentrySTDinterwordspacing

\bibitem{zou_hybrid2_2024}
\BIBentryALTinterwordspacing
B.~J. Zou, M.~E. Levine, D.~P. Zaharieva, R.~Johari, and E.~B. Fox, ``Hybrid\${\textasciicircum}2\$ {Neural} {ODE} {Causal} {Modeling} and an {Application} to {Glycemic} {Response},'' June 2024, arXiv:2402.17233 [cs, stat].
\BIBentrySTDinterwordspacing

\bibitem{lohmiller_contraction_1998}
\BIBentryALTinterwordspacing
W.~Lohmiller and J.-J.~E. Slotine, ``\BIBforeignlanguage{en}{On {Contraction} {Analysis} for {Non}-linear {Systems}},'' \emph{\BIBforeignlanguage{en}{Automatica}}, vol.~34, no.~6, pp. 683--696, June 1998.
\BIBentrySTDinterwordspacing

\bibitem{niethammer_parameter_2001}
\BIBentryALTinterwordspacing
M.~Niethammer, P.~H. Menold, and F.~Allgöwer, ``Parameter and {Derivative} {Estimation} for {Nonlinear} {Continuous}-{Time} {System} {Identification},'' \emph{IFAC Proceedings Volumes}, vol.~34, no.~6, pp. 663--668, July 2001.
\BIBentrySTDinterwordspacing

\bibitem{kidger_neural_2022}
\BIBentryALTinterwordspacing
P.~Kidger, ``On {Neural} {Differential} {Equations},'' Feb. 2022, arXiv:2202.02435 [cs, math, stat].
\BIBentrySTDinterwordspacing

\end{thebibliography}
\bibliographystyle{IEEEtran}

\end{document}